\documentclass{article}

\usepackage{microtype}
\usepackage{graphicx}
\usepackage{subcaption}
\usepackage{booktabs} 
\usepackage{diagbox} 
\usepackage{hyperref}

\PassOptionsToPackage{numbers, compress}{natbib}
\usepackage{arxiv}

\usepackage{color}
\newcounter{nbdrafts}
\makeatletter
\newcommand{\checknbdrafts}{
\ifnum \thenbdrafts > 0
\@latex@warning@no@line{*WARNING* The document contains \thenbdrafts \space draft note(s)}
\fi}

\makeatother

\usepackage{algorithm} 
\usepackage{algorithmic}  
\usepackage[ruled, algo2e]{algorithm2e}
\usepackage{amsmath}
\usepackage{amssymb}
\usepackage{amsfonts}
\usepackage{amsthm}
\usepackage{subcaption}
\usepackage{times}
\usepackage{latexsym}
\usepackage{breqn}
\usepackage{float}
\usepackage{multirow}
\usepackage{adjustbox}
\usepackage{xcolor}

\newtheorem{theorem}{Proposition}
\newcommand{\bigO}[1]{\mathcal{O}\left(#1\right)}
\newcommand{\softmax}[1]{\sfx\left(#1\right)}
\newcommand{\expo}[1]{\exp\left(#1\right)}

\DeclareMathOperator*{\sfx}{softmax}
\newcommand{\R}{\mathbb{R}}

\newcommand{\norm}[1]{\left\|#1\right\|}
\newcommand{\abs}[1]{\left|#1\right|}
\newcommand{\un}[1]{\underline{#1}}
\renewcommand{\cite}{\citep}
\newcommand{\boldparagraph}[1]{\vspace{0.2cm}\noindent{\bf #1:}}
\newcommand\FramedBox[3]{%
  \setlength\fboxsep{3pt}
    \fbox{\parbox[t][#1][c]{#2}{\centering\footnotesize #3}}}

\title{Fast Transformers with Clustered Attention}
\author{
    Apoorv Vyas$^{1 2}$
    \quad Angelos Katharopoulos$^{1 2}$
    \quad Fran\c{c}ois Fleuret$^{2 3}$ \thanks{Work done at Idiap} \\
    $^{1}$Idiap Research Institute, Switzerland \\
    $^{2}$Ecole Polytechnique F\'ed\'erale de Lausanne, Switzerland\\
    $^{3}$University of Geneva, Switzerland \\
    \texttt{firstname.lastname@idiap.ch}
}

\begin{document}

\maketitle
\begin{abstract}
Transformers have been proven a successful model for a variety of tasks in
sequence modeling. However, computing the attention matrix, which is their key
component, has quadratic complexity with respect to the sequence length, thus
making them prohibitively expensive for large sequences. To address this, we
propose \emph{clustered attention}, which instead of computing the attention
for every query, groups queries into clusters and computes attention just for
the centroids. To further improve this approximation, we use the computed
clusters to identify the keys with the highest attention per query and compute
the exact key/query dot products. This results in a model with linear
complexity with respect to the sequence length for a fixed number of clusters.
We evaluate our approach on two automatic speech recognition datasets and show
that our model consistently outperforms vanilla transformers for a given
computational budget.  Finally, we demonstrate that our model can approximate
arbitrarily complex attention distributions with a minimal number of clusters
by approximating a pretrained BERT model on GLUE and SQuAD benchmarks with only
25 clusters and no loss in performance.
\end{abstract}

\section{Introduction}

Sequence modelling is a fundamental task of machine learning, integral in a
variety of applications such as neural machine translation
\cite{bahdanau2014neural}, image captioning \cite{xu2015show}, summarization
\cite{maybury1999advances}, automatic speech recognition \cite{dong2018speech}
and synthesis \cite{oord2016wavenet} etc. Transformers
\cite{vaswani2017attention} have been proven a powerful tool significantly
advancing the state-of-the-art for the majority of the
aforementioned tasks. In particular, transformers employ self-attention that
allows them to handle long sequences without the vanishing-gradient problem
inherent in RNNs \cite{hochreiter2001gradientflow, arjovsky2016unitary}.

Nonetheless, despite their impressive performance, the use of self-attention
comes with computational and memory requirements that scale quadratic to the
sequence length, limiting their applicability to long sequences. The quadratic
complexity becomes apparent if we consider the core mechanism of self-attention,
namely splitting the input sequence into queries and keys and then each query
attending to all keys. To this end, recently, there has been an increasing
interest for developing methods that address this limitation
\cite{dai2019transformer, sukhbaatar2019adaptive, child2019generating,
kitaev2020reformer}.

These methods can be broadly categorized into two distinct lines of work, those
that focus on improving the asymptotic complexity of the self-attention
computation \cite{child2019generating, lee2019set, kitaev2020reformer, roy2020efficient} and those that aim
at developing techniques that make transformers applicable to longer sequences
without addressing the quadratic complexity of self-attention
\cite{dai2019transformer, sukhbaatar2019adaptive}. The former limits the amount
of keys that each query attends to, thus reducing the asymptotic
complexity. The latter increases the length of the sequence that a
transformer can attend to without altering the underlying complexity of the
self-attention mechanism.

In this work, we propose \emph{clustered attention} which is a fast
approximation of self-attention. Clustered attention makes use of similarities
between queries and groups them in order to reduce the computational cost. In
particular, we perform fast clustering using locality-sensitive hashing and
K-Means and only compute the attention once per cluster.  This results in
linear complexity for a fixed number of clusters
(\S~\ref{subsec:grouped_attention}).
In addition, we showcase that we can further improve the quality of our
approximation by separately considering the keys with the highest attention per
cluster (\S~\ref{subsec:improved_attn}). Finally, we provide theoretical bounds
of our approximation quality with respect to the full attention
(\S~\ref{subsubsec:approximation_quality}, \S~\ref{subsec:quality_improved})
and show that our model can be applied for inference of pre-trained
transformers with minimal loss in performance.

We evaluate our model on two automatic speech recognition datasets and showcase
that clustered attention consistently achieves better performance than vanilla
attention when the computational budget is equalized. Moreover, we demonstrate
that our proposed attention can approximate a pretrained BERT model on the
popular GLUE and SQuAD benchmarks with only 25 clusters and without loss in
performance.

\section{Related Work}

In this section, we discuss the most relevant works on scaling transformers to
larger sequences. We start by presenting approaches that aim to speed up the
attention computation in general.  Subsequently, we discuss approaches that
speed up transformers without changing the complexity of the attention layer
and finally, we summarize the most related works on improving the asymptotic
complexity of the attention layer in transformer models.

\subsection{Attention Improvements Before Transformers}

Attention has been an integral component of neural networks for sequence
modelling for several years \cite{bahdanau2014neural, xu2015show,
chan2016listen}. However, its quadratic complexity with respect to the sequence
length hinders its applicability on large sequences.

Among the first attempts to address this was the work of
\citet{britz2017efficient} that propose to aggregate the information of the
input sequence into fewer vectors and perform attention with these fewer
vectors, thus speeding up the attention computation and reducing the memory
requirements. However, the input aggregation is performed using a learned but fixed
matrix that remains constant for all sequences, hence significantly
limiting the expressivity of the model.  Similarly, \citet{chiu2017monotonic}
limit the amount of accessible elements to the attention, by
attending monotonically from the past to the future. Namely, if timestep $i$
attends to position $j$ then timestep $i+1$ cannot attend to any of the earlier
positions.  Note that in order to speed up the attention computation, the above
methods are limiting the number of elements that each layer attends to.
Recently, some of these approaches have also been applied in the context of
transformers \cite{Ma2020Monotonic}.

\subsection{Non-asymptotic Improvements} \label{sec:related-non-asym}

In this section, we summarize techniques that seek to apply transformers to long
sequences without focusing on improving the quadratic complexity of
self-attention. The most important are Adaptive Attention Span Transformers
\cite{sukhbaatar2019adaptive} and Transformer-XL \cite{dai2019transformer}.

\citet{sukhbaatar2019adaptive} propose to limit the self-attention context to
the closest samples (attention span), in terms of relative distance with
respect to the time step, thus reducing both the time and memory requirements of
self-attention computation. This is achieved using a masking function with learnable parameters
that allows the network to increase the attention span if necessary.
Transformer-XL \cite{dai2019transformer}, on the other hand, seeks to increase
the effective sequence length by introducing segment-level recurrent training,
namely splitting the input into segments and attending jointly to the previous
and the current segment. The above, combined with a new relative positional
encoding results in models that attend to more distant positions than the
length of the segment used during training.

Although both approaches have been proven effective, the underlying
limitations of self-attention still remains. Attending to an element
that is $N$ timesteps away requires $\bigO{N^2}$ memory and computation.  In
contrast, our model trades-off a small error in the computation of the full
attention for an improved \emph{linear} asymptotic complexity. This makes
processing long sequences possible.

\subsection{Improvements in Asymptotic Complexity}

\citet{child2019generating} factorize the self-attention mechanism in local and
strided attention. The local attention is computed between the $C$ nearest
positions and the strided attention is computed between positions that are $C$
steps away from each other. When $C$ is set to $\sqrt{N}$ the total asymptotic
complexity becomes $\bigO{N \sqrt{N}}$ both in terms of memory and computation
time. With the aforementioned factorization, two self-attention layers are
required in order for any position to attend to any other position.
In addition, the factorization is fixed and data independent. This makes it
intuitive for certain signals (e.g. images), however in most cases it is
arbitrary. In contrast, our method automatically groups the input queries that
are similar without the need for a manually designed factorization. Moreover,
in our model, information flows always from every position to every other
position.


Set Transformers \cite{lee2019set} compute attention between the input sequence
$X$, of length $N$ and a set of trainable parameters, $I$, called {inducing
points} to get a new sequence $H$, of length $M << N$. The new sequence $H$ is
then used to compute the attention with $X$ to get the output representation.
For a fixed $M$, the asympotic complexity becomes linear with respect to the
sequence length.  Inducing points are expected to encode some global structure
that is task specific. However, this introduces additional model parameters for
each attention layer. In contrast to this, we use clustering to project the
input to a fixed sequence of smaller length without any increase in the number
of parameters. Moreover, we show that not only our method has the same
asymptotic complexity, it can also be used to speed up inference of pretrained
models without additional training.

Recently, \citet{kitaev2020reformer} introduced Reformer. A method that
groups positions based on their similarity using locality-sensitive
hashing (LSH) and only computes the attention within groups. For groups of
fixed size, the asymptotic complexity of Reformer becomes linear with respect
to the sequence length. Note that Reformer constrains the queries and keys of
self-attention to be equal. As a result, it cannot be applied to neural machine
translation, image captioning or memory networks, or generally any application
with heterogenous queries and keys. In addition, as it uses hash collisions to
form groups it can only handle a small number of bits, thus significantly
reducing the quality of the grouping. Instead, our method uses clustering to
group the queries, resulting in significantly better groups compared to
hash collisions.

\section{Scaling Attention with Fast Clustering}

In this section, we formalize the proposed method for approximate softmax
attention. In  \S~\ref{subsec:vanilla_attn}, we first discuss the attention
mechanism in \emph{vanilla transformers} and present its computational
complexity. We then introduce \emph{clustered attention} in
\S~\ref{subsec:grouped_attention} and show that for queries close in the
Euclidean space, the attention difference can be bounded by the distance
between the queries. This property allows us to reduce the computational
complexity by clustering the queries. Subsequently, in
\S~\ref{subsec:improved_attn} we show that we can further improve the
approximation by first extracting the top-$k$ keys with the highest attention
per cluster and then computing the attention on these keys separately for each
query that belongs to the cluster. A graphical illustration of
our method is provided in the supplementary material.

\subsection{Vanilla Attention}\label{subsec:vanilla_attn}

For any sequnce of length $N$, the standard attention mechanism that is used in
transformers is the dot product attention introduced by
\citet{vaswani2017attention}. Following standard notation, we define the
attention matrix $A \in \R^{N \times N}$ as,
\begin{equation}
    A = \softmax{\frac{Q K^T}{\sqrt{D_k}}},
    \label{eq:attention}
\end{equation}
where $Q \in \R^{N \times D_k}$ denotes the \emph{queries} and $K \in \R^{N
\times D_k}$ denotes the \emph{keys}. Note that $\softmax{\cdot}$ is applied
row-wise. Using the attention weights $A$ and the values $V \in \R^{N \times
D_v}$, we compute the new values $\hat{V}$ as follows,
\begin{equation}
    \hat{V} = A V.
    \label{eq:new_values}
\end{equation}
An intuitive understanding of the attention, as described above, is that given
$Q, K, V$ we create new values $\hat{V}$ as the weighted average of the
old ones, where the weights are defined by the attention matrix $A$.
Computing equation \ref{eq:attention} requires $\bigO{N^2 D_k}$ operations and
the weighted average of equation \ref{eq:new_values} requires $\bigO{N^2 D_v}$.
This results in an asymptotic complexity of $\bigO{N^2 D_k + N^2 D_v}$.

\subsection{Clustered Attention} \label{subsec:grouped_attention}

Instead of computing the attention matrix for all queries, we group
them into $C$ clusters and compute the attention only for these clusters. Then,
we use the same attention weights for queries that belong to the same cluster.
As a result, the attention computation now becomes $\bigO{N C D_k}$, where $C
\ll N$. 

More formally, let us define $S \in \{0,1\}^{N \times C}$, a partitioning of
the queries $Q$ into $C$ non-overlapping clusters, such that, $S_{ij} = 1$, if
the $i$-th query $Q_i$ belongs to the $j$-th cluster and $0$ otherwise.  Using
this partitioning, we can now compute the \emph{clustered attention}. First, we
compute the cluster centroids as follows,
\begin{align}
    Q^c_{j} &= \frac{\sum_{i=1}^{N}S_{ij}Q_{i}}{\sum_{i=1}^{N}S_{ij}},
    \label{eq:centroid}
\end{align}
where $Q^c_j$ is the centroid of the $j$-th cluster. Let us denote $Q^c \in \R^{C
\times D_k}$ as the centroid matrix. Now, we can compute the clustered
attention as if $Q^c$ were the queries. Namely, we compute the clustered
attention matrix $A^c \in \R^{C \times N}$
\begin{equation}
    A^c = \softmax{\frac{Q^c K^T}{\sqrt{D_k}}} \label{eq:clust_attn}
\end{equation}
and the new values $\hat{V}^c \in \R^{C \times D_v}$
\begin{equation}
    \hat{V}^c = A^c V. \label{eq:clust_vals}
\end{equation}
Finally, the value of the $i$-th query becomes the value of its closest
centroid, namely,
\begin{equation}
    \hat{V}_i = \sum_{j=1}^C S_{ij} \hat{V}^c_j. \label{eq:broadcast}
\end{equation}

From the above analysis, it is evident that we only need to compute the
attention weights and the weighted average of the values \emph{once per
cluster}. Then, we can broadcast the same value to all queries belonging to the
same cluster. This allows us to reduce the number of dot products from $N$ for
each query to $C$ for each cluster, which results in an asymptotic complexity
of $\bigO{N C D_k} + \bigO{C N D_v}$.

Note that in practice, we use multi-head attention, this means that two queries
belonging to the same cluster can be clustered differently in another attention
head. Moreover, the output of the attention layer involves residual
connections.  This can cause two queries belonging to the same cluster to have
different output representations. The combined effect of residual connections
and multi-head attention allows new clustering patterns to emerge in subsequent
layers.

\subsubsection{Quality of the approximation} \label{subsubsec:approximation_quality}

From the above, we show that grouping queries into clusters can speed-up the
self-attention computation. However, in the previous analysis, we do not
consider the effects of clustering on the attention weights $A$.  To address
this, we derive a bound for the approximation error. In particular, we show
that the difference in attention can be bounded as a function of the Euclidean
distance between the queries.

\begin{theorem} \label{thm:clustered_quality}
    Given two queries $Q_i$ and $Q_j$ such that $\norm{Q_i - Q_j}_2
    \leq\epsilon$,
    \begin{equation}
        \begin{aligned}
            \norm{\softmax{Q_i K^T} - \softmax{Q_j K^T}}_2 \leq
                 \epsilon \norm{K}_2 \,,
        \end{aligned}
    \end{equation}
    where $\norm{K}_2$ denotes the spectral norm of $K$.
\end{theorem}
\begin{proof}
Given that $\softmax{\cdot}$ has Lipschitz constant less than 1
\cite{gao2017properties},
\begin{equation}
    \begin{aligned}
     & \norm{\softmax{Q_i K^T} - \softmax{Q_j K^T}}_2 \\
     & \quad \leq \norm{Q_i K^T - Q_j K^T}_2 \\
     & \quad \leq \epsilon \norm{K}_2
     \end{aligned}
\end{equation}
\end{proof}

Proposition \ref{thm:clustered_quality} shows that queries that are close in
Euclidean space have similar attention distributions. As a result, the error in
the attention approximation for the $i$-th query assigned to the $j$-th
cluster can be bounded by its distance from the cluster centroid $Q^c_{j}$.

\subsubsection{Grouping the Queries}
\label{subsubsec:query_group}

From the discussion, we have shown that given a representative set of queries,
we can approximate the attention with fewer computations. Thus, now the problem
becomes finding this representative set of queries. K-Means clustering
minimizes the sum of squared distances between the cluster members, which would
be optimal given our analysis from \S~\ref{subsubsec:approximation_quality}.
However, for a sequence of length $N$ one iteration of Lloyd's algorithm for
the K-Means optimization problem has an asymptotic complexity $\bigO{N
C D_k}$. To speed up the distance computations, we propose to use
\emph{Locality-Sensitive Hashing} (LSH) on the queries and then K-Means in
Hamming space. In particular, we use the sign of random projections
\cite{shrivastava2014asymmetric} to hash the queries followed by K-Means
clustering with hamming distance as the metric.  This results in an asymptotic
complexity of $\bigO{N C L + C B L + N D_k B}$, where $L$ is the number of
Lloyd iterations and $B$ is the number of bits used for hashing.

%
\subsection{Improving clustered attention} \label{subsec:improved_attn}

In the previous section, we show that clustered attention provides a fast
approximation for softmax attention. In this section, we discuss how this
approximation can be further improved by considering separately the keys with
the highest attention. To intuitively understand the importance of the above,
it suffices to consider a scenario where a key with low attention for some
query gets a high attention as approximated with the cluster centroid. This can
happen when the number of clusters are too low or due to the convergence
failure of K-Means. For the clustered attention, described in
\S~\ref{subsec:grouped_attention}, this introduces significant error in the
computed value. The variation discussed below addresses such limitations.

After having computed the clustered attention $A^c$ from equation
\ref{eq:clust_attn}, we find the $k$ keys with the highest attention for each
cluster. The main idea then is to improve the attention approximation on these
top-$k$ keys for each query that belongs to the cluster. To do so, we first
compute the dot product attention as defined in equation \ref{eq:attention} on
these top-$k$ keys for all queries belonging to this cluster. For any query,
the computed attention on these top-$k$ keys will sum up to one. This means
that it cannot be directly used to substitute the
clustered-attention on these keys. To address this, before substition, we scale
the computed attention by the total probability mass assigned by the clustered
attention to these top-$k$ keys.

More formally, we start by introducing $T \in \{0, 1\}^{C \times N}$, where
$T_{ji} = 1$ if the $i$-th key is among the top-$k$ keys for the $j$-th cluster
and 0 otherwise. We can then compute the probability mass, let it be
$\hat{m}_j$, of the top-$k$ keys for the $j$-th cluster, as follows
\begin{align}
    \hat{m}_j = \sum_{i=1}^{N}T_{ji}A^c_{ji}.
\end{align}
Now we formulate an improved attention matrix approximation $A^t \in \R^{N
\times N}$ as follows
\begin{equation}
    A^t_{il} = \begin{cases}
                   \frac{\hat{m}_j \expo{Q_i K_l^T}}
                                            {\sum_{r=1}^N T_{jr} \expo{Q_i K_r^T}}
                                            & \text{if } T_{jl} = 1 \\
                   A^c_{jl} & \text{otherwise}
               \end{cases}. \label{eq:improved_attention}
\end{equation}
Note that in the above, $i$ denotes the $i$-th query belonging to the $j$-th
cluster and $\sqrt{D_k}$ is ommited for clarity.
In particular, equation \ref{eq:improved_attention} selects the clustered attention of
equation \ref{eq:clust_attn} for keys that are not among the top-$k$ keys for a
given cluster. For the rest, it redistributes the mass $\hat{m}_j$ according to
the dot product attention of the queries with the top-$k$ keys. The
corresponding new values, $\hat{V} \in \R^{N \times D_{v}}$, are a simple
matrix product of $A^t$ with the values,
\begin{align}
   \hat{V} = A^t V. \label{eq:improved_values_slow}
\end{align}
Equation \ref{eq:improved_values_slow} can be decomposed into
clustered attention computation and two sparse dot products, one for every
query with the top-$k$ keys and one for the top-$k$ attention weights with the
corresponding values. This adds $\bigO{N k \max\left(D_k, D_v\right)}$ to
the asymptotic complexity of the attention approximation of equation
\ref{eq:clust_attn}.


\subsubsection{Quality of the approximation}\label{subsec:quality_improved}

In the following, we provide proof that improved clustered attention (eq.
\ref{eq:improved_attention}) is a direct improvement over the clustered
attention (eq. \ref{eq:clust_attn}), in terms of the $L_1$ distance from the
attention matrix $A$.

\begin{theorem} \label{thm:improved_attention}
For the $i$-th query belonging to the $j$-th cluster, the improved clustered
attention $A^t_i$ and clustered attention $A^c_j$ relate to the full attention
$A_i$ as follows,
\begin{equation}
    \norm{A^t_i - A_i}_1 \leq \norm{A^c_j - A_i}_1
    \label{eq:improvement_proof}
\end{equation}
\end{theorem}
Due to lack of space, the proof of the above proposition is presented in the
supplementary material. From equation \ref{eq:improvement_proof} it becomes
evident that improved clustered attention will always approximate the full
attention better compared to clustered attention.

\section{Experiments}

In this section, we analyze experimentally the performance of our proposed
method. Initially, we show that our model outperforms our baselines
for a given computational budget on a real-world sequence to sequence task,
namely automatic speech recognition on two datasets, the Wall Street Journal
dataset (\S~\ref{subsec:asr_wsj}) and the Switchboard dataset
(\S~\ref{subsec:asr_swbd}). Subsequently, in \S~\ref{sec:approx}, we
demonstrate that our model can approximate a pretrained BERT model
\cite{liu2019roberta} on the GLUE \cite{wangsmhlb19} and SQuAD
\cite{rajpurkar2018know} benchmarks with minimal loss in performance even when
the number of clusters is less than one tenth of the sequence length. Due to
lack of space, we also provide, in the supplementary material, a thorough
benchmark that showcases the linear complexity of \emph{clustered attention}
and an ablation study regarding how the number of clusters scales with respect
to the sequence length.

We compare our model with the vanilla transformers \cite{vaswani2017attention},
which we refer to as \textbf{full} and the Reformer \cite{kitaev2020reformer},
which we refer to as \textbf{lsh-X}, where $X$ denotes the rounds of hashing.
We refer to \emph{clustered attention}, introduced in
\S~\ref{subsec:grouped_attention}, as \textbf{clustered-X} and to
\emph{improved clustered attention}, introduced in
\S~\ref{subsec:improved_attn}, as \textbf{i-clustered-X}, where $X$ denotes
the number of clusters. Unless mentioned otherwise we use $k=32$ for the
top-$k$ keys with improved clustered.

All experiments are conducted using NVidia GTX 1080 Ti with $11$GB of memory
and all models are implemented in PyTorch \cite{paszke2019pytorch}.
For Reformer we use a PyTorch port of the published code. Note that we do not
use reversible layers since it is a technique that could be applied to all
methods. Our PyTorch code can be found at
\url{https://clustered-transformers.github.io}.

\begin{figure*}
    \begin{subfigure}{\linewidth}
        \centering
        \includegraphics[width=0.6\textwidth]{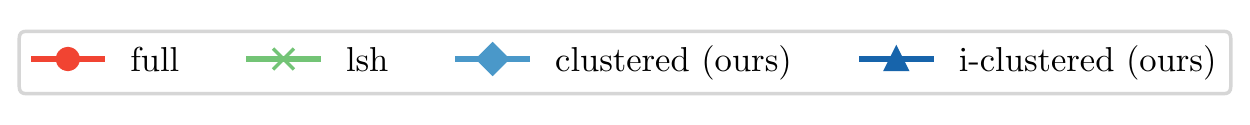}
    \end{subfigure}
    \begin{subfigure}{0.49\linewidth}
        \includegraphics[width=\textwidth]{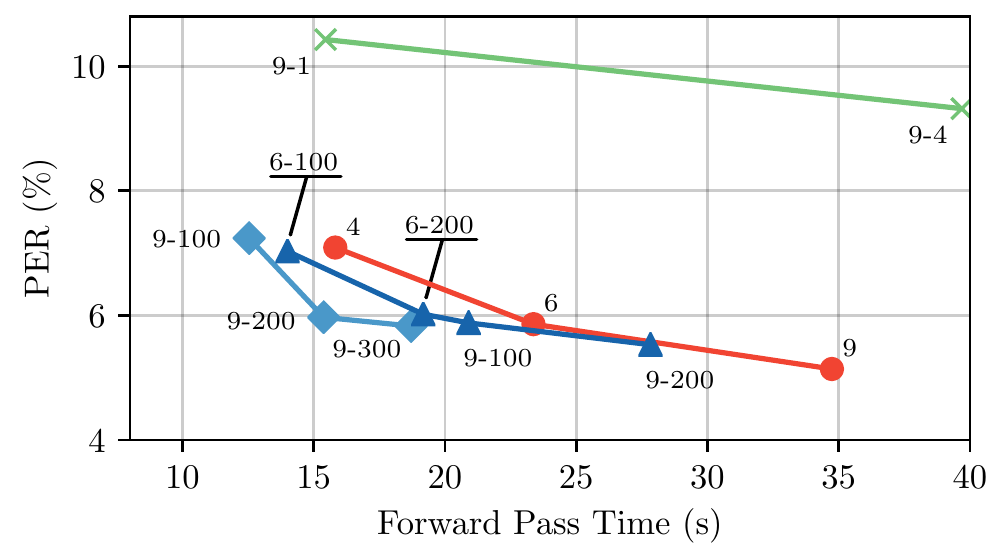}
        \caption{Wall Street Journal} \label{fig:asr_wsj}
    \end{subfigure}
    \begin{subfigure}{0.49\linewidth}
        \includegraphics[width=\textwidth]{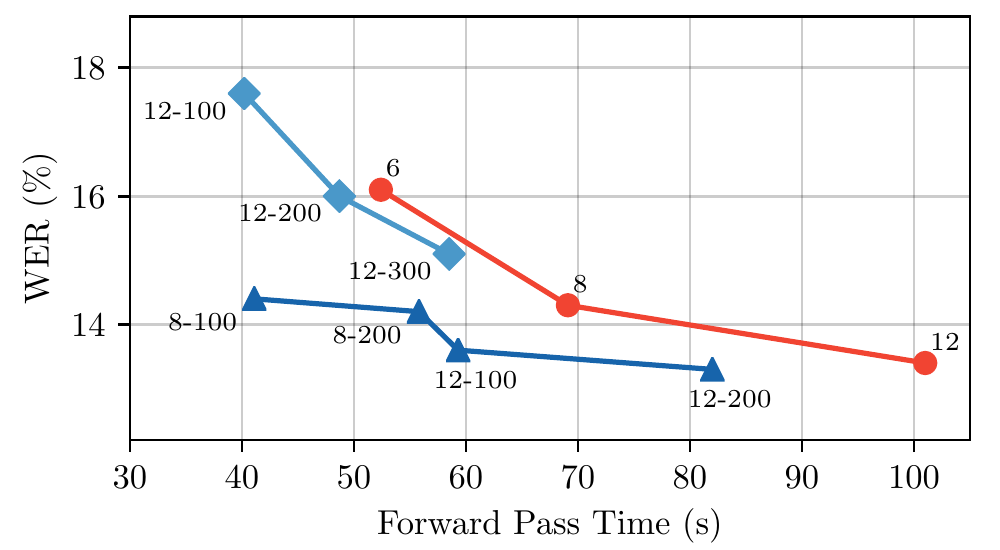}
        \caption{Switchboard} \label{fig:asr_swbd}
    \end{subfigure}
    \caption{We compare the achieved performance of various transformer models
             under an equalized computational budget. The numbers near the
             datapoints denote the number of layers and number of clusters or
             hashing rounds where applicable. i-clustered is consistently
             better than all baselines for a given computational budget both in
             WSJ and Switchboard datasets. The details can be found in
             \S~\ref{subsec:asr_wsj} and \S~\ref{subsec:asr_swbd}
             respectively.}
    \label{fig:asr}
\end{figure*}

\subsection{Evaluation on Wall Street Journal (WSJ)} \label{subsec:asr_wsj}

In our first experiment, we employ the Wall-Street Journal dataset
\cite{douglas1992wsj}. The input to all transformers is $40$-dimensional
filter-bank features with fixed positional embeddings. We train
using Connectionist Temporal Classification (CTC) \cite{graves2006ctc} loss
with phonemes as ground-truth labels. The approximate average and maximum
sequence lengths for the training inputs are $780$ and $2500$ respectively.

\boldparagraph{Speed Accuracy Trade-off}
We start by comparing the performance of our proposed model with various
transformer variants under an equalized computational budget. To this end, we
train \emph{full} with $4$, $6$ and $9$ layers to get a range of the required
computation time and achieved \emph{phone error rate} (PER). Similarly, we
train \emph{i-clustered} with $6$ and $9$ layers. Both models are trained
with $100$ and $200$ clusters.
We also train \emph{clustered} with $9$ layers, and $100$, $200$ and $300$
clusters. Finally, we train Reformer with $9$ layers, and $1$ and $4$ hashing
rounds. We refer the reader to our supplementary for the specifics of all
transformer architectures as well as their training details.
In figure \ref{fig:asr_wsj}, we plot the achieved PER on the validation set with
respect to the required time to perform a full forward pass. Our
\emph{i-clustered} achieves lower PER than all other baselines for a given
computational budget.

\boldparagraph{Approximation Quality}
To assess the approximation capabilities of our method, we train different
transformer variants on the aforementioned task and evaluate them using other
self-attention implementations during inference. As the Reformer requires the
queries to be identical to the keys to evaluate its approximation ability we
also train a full attention model with shared queries and keys, which we refer
to as \textbf{shared-full}. Note that both clustered attention and improved
clustered attention can be used for approximating shared-full, simply by
setting keys to be equal to queries.  Table~\ref{tab:asr_approximation}
summarizes the results.  We observe that improved clustered attention (7-8
rows) achieves the lowest phone error rate in every comparison. This implies
that it is the best choice for approximating pre-trained models. In addition,
we also note that as the number of clusters increases, the approximation
improves as well.
\bgroup
\renewcommand{\arraystretch}{1.1}
\begin{table*}[h]
    \begin{center}
    \scalebox{0.9}{
    \begin{tabular}{ll||cccccc}
        \multicolumn{2}{c}{} & \multicolumn{6}{c}{\textbf{Train with}} \\
        & \multicolumn{1}{c}{} & full & shared-full & lsh-1 & lsh-4
            & clustered-100 & i-clustered-100 \\
        \toprule
        \multirow{9}{*}{\rotatebox[origin=c]{90}{\textbf{Evaluate with}}}
        & full               &  \un{5.14} &  -        &  -         &  -         &  7.10      &  5.56      \\
        & shared-full        &  -         & \un{6.57} & 25.16      & 41.61      &   -        &   -        \\
        & lsh-1              &  -         & 71.40     & \un{10.43} & 13.76      &   -        &   -        \\
        & lsh-4              &  -         & 64.29     &  9.35      & \un{9.33}  &   -        &   -        \\
        & clustered-100      & 44.88      & 40.86     & 68.06      & 66.43      &  \un{7.06} & 18.83      \\
        & clustered-200      & 21.76      & 25.86     & 57.75      & 57.24      &  6.34      &  8.95      \\
        & i-clustered-100 &  9.29      & 13.22     & 41.65      & 48.20      &  8.80      &  \un{5.95} \\
        & i-clustered-200 &  6.38      &  8.43     & 30.09      & 42.43      &  7.71      &  5.60      \\
        \cmidrule{2-8}
        & oracle-top      & 17.16      & 77.18     & 43.35      & 59.38      & 24.32      &  6.96      \\
        \bottomrule
    \end{tabular}
    }
    \end{center}
    \caption{
        We report validation phone error rate (PER) on the WSJ dataset
        (\S~\ref{subsec:asr_wsj}).  We train with one model and evaluate with
        another to assess the approximation abilities of different models.
        \un{Underline} denotes training and testing with the same model.
        Improved cluster (rows 7-8) approximates the full and the shared-full
        significantly better than all the other fast attention methods.
    }
    \label{tab:asr_approximation}
\end{table*}
\egroup
Furthermore, to show that the top keys alone are not sufficient for
approximating \emph{full}, we also compare with an attention variant, that for each
query only keeps the $32$ keys with the highest attention. We refer to the
latter as \textbf{oracle-top}. We observe that oracle-top achieves
significantly larger phone error rate than improved clustered in all cases.
This implies that improved clustered attention also captures the significant
long tail of the attention distribution.

\boldparagraph{Convergence Behaviour}
In Table \ref{tab:asr_per}, we report the required time per epoch as well as
the total training time for all transformer variants with 9 layers. For
completeness, we also provide the corresponding phone error rates on the test
set. We observe that clustered attention is more than two times faster than
full (per epoch) and achieves significantly lower PER than both Reformer
variants (lsh-$1$ and lsh-$4$). Improved clustered is the only method that is
not only faster per epoch but also in total wall-clock time required to
converge.
\begin{table}[h!]
    \centering
    \begin{tabular}{r|ccccc}
        \toprule
        \multicolumn{1}{c}{\,} & full & lsh-1 & lsh-4 & clustered-100 & i-clustered-100 \\
        \midrule
        PER (\%) & 5.03 & 9.43 & 8.59 & 7.50 & 5.61 \\
        Time/Epoch (s) & 2514 & 1004 & 2320 & 803 & 1325 \\
        Convergence Time (h) & 87.99 & 189.64 & 210.09 & 102.15 & 72.14 \\
        \bottomrule
    \end{tabular}
    \vspace{0.5em}
    \caption{We report the test PER, the time per training epoch (in seconds)
             and the wall-clock time required for the convergence of each model
             (in hours).}
    \label{tab:asr_per}
\end{table}

\subsection{Evaluation on Switchboard} \label{subsec:asr_swbd}

We also evaluate our model on the Switchboard dataset
\cite{godfrey1992switchboard}, which is a 
collection of $2,400$ telephone conversations on common topics among $543$
strangers. All transformers are trained with lattice-free MMI loss
\cite{povey2016purely} and as inputs we use $80$-dimensional filter-bank
features with fixed positional embeddings. The average input sequence length is
roughly $534$ and the maximum sequence length is approximately $3850$.
Details regarding the transformer architectures as well as their training
details are provided in the supplementary.

\boldparagraph{Speed Accuracy Trade-off}
Similar to \S~\ref{subsec:asr_wsj}, we compare the performance of various
transformer models given a specific computational budget.
To this end, we train
\emph{full} with $6, 8$ and $12$ layers. Similarly, we train \emph{i-clustered}
with $8$ and $12$ layers; both with $100$ and $200$
clusters. Finally, we also train \emph{clustered} with $12$ layers, and $100,
200$ and $300$ clusters. In figure \ref{fig:asr_swbd}, we plot the achieved
word error rate (WER) in the validation set of Switchboard with respect to the
required time to perform a full forward pass. Our \emph{i-clustered} is
consistently better than \emph{full} for a given computational budget. In
particular, for a budget of approximately $50$ seconds, improved clustered
achieves more than $2$ percentage points lower WER. Furthermore, we note that
it is consistently better than clustered attention for all computational
budgets.

\boldparagraph{Convergence Behaviour}
Table \ref{tab:asr_wer} summarizes the computational cost of training the
transformer models with 12 layers in the Switchboard dataset as well as the WER
in the test set. We observe that due to the larger sequences in this dataset
both clustered and i-clustered are faster to train per epoch and with respect
to total required wall-clock time.
\begin{table}[h!]
    \centering
    \begin{tabular}{r|ccc}
        \toprule
        \multicolumn{1}{c}{\,} & full & clustered-100 & i-clustered-100 \\
        \midrule
        WER (\%) & 15.0 & 18.5 & 15.5 \\
        Time/Epoch (h) & 3.84 & 1.91 &  2.57 \\
        Convergence Time (h) & 228.05 & 132.13 & 127.44 \\
        \bottomrule
    \end{tabular}\vspace{0.5em}
    \caption{We report the test set WER, the time per training epoch (in hours)
             and the wall-clock time required for the convergence of each model
             (in hours).}
    \label{tab:asr_wer}
\end{table}

\subsection{RoBERTa Approximation} \label{sec:approx}

To highlight the ability of our model to approximate arbitrarily complicated
attention distributions, we evaluate our proposed method on the approximation
of a fine-tuned RoBERTa model \cite{liu2019roberta} on the GLUE
\cite{wangsmhlb19} and SQuAD \cite{rajpurkar2018know} benchmarks. In
particular, we evaluate on $10$ different tasks, among which there are tasks
such as question answering (SQuAD) and textual entailment (RTE), which exhibit
arbitrary and sparse attention patterns. We refer the reader to
\citet{wangsmhlb19, rajpurkar2018know} for a detailed analysis of all tasks.

For the GLUE tasks, the maximum sequence length is 128 while for SQuAD, it
is 384. For each task, we use $25$ clusters for
approximation which is less than $20\%$ and $10\%$ of the input sequence length
for GLUE and SQuAD tasks respectively. In Table \ref{tab:bert_approx}, we
summarize the performance per task. We observe that improved clustered performs
as well as the full transformer in all tasks but SQuAD, in which it is only
marginally worse. Moreover, we note that clustered performs significantly worse
in tasks that require more complicated attention patterns such as SQuAD and RTE.
For inference time, \emph{full} was faster than the \emph{clustered} attention
variants due to short sequence lengths.

\begin{table*}[h!]
    \resizebox{\textwidth}{!}{
    \begin{tabular}{rcccccccccc}
        & CoLA & MNLI & MRPC & QNLI & QQP & RTE & SST-2 & STS-B & WNLI & SQuAD \\
        \toprule
        full           & 0.601 & 0.880 & 0.868 & 0.929 & 0.915 & 0.682 & 0.947 & 0.900 &
            0.437 & 0.904 \\
        clustered-25   & 0.598 & 0.794 & 0.436 & 0.746 & 0.894 & 0.498 & 0.944 & 0.789 &
            0.437 & 0.006 \\
        i-clustered-25 & 0.601 & 0.880 & 0.873 & 0.930 & 0.915 & 0.704 & 0.947 & 0.900 &
            0.437 & 0.876\\
    \end{tabular}}
    \caption{We report the performance on GLUE and SQuAD benchmarks. Following
             common practice, we report accuracy for all tasks except STS-B and
             SQuAD, where we report Pearson correlation and F1-score
             respectively. For all metrics higher is better.}
    \label{tab:bert_approx}
\end{table*}

\section{Conclusions}

We have presented \emph{clustered attention} a method that approximates vanilla
transformers with significantly lower computational requirements. In
particular, we have shown that our model can be up to $2 \times$ faster during
training and inference with minimal loss in performance. In contrast to recent
fast variations of transformers, we have also shown that our method can
efficiently approximate pre-trained models with full attention while retaining
the linear asymptotic complexity.

The proposed method opens several research directions towards applying
transformers on long sequence tasks such as music generation, scene flow
estimation etc. We consider masked language modeling for long texts to be of
particular importance, as it will allow finetuning for downstream tasks
that need a context longer than the commonly used 512 tokens.

\section*{Broader Impact}

This work contributes towards the wider adoption of transformers by reducing
their computational requirements; thus enabling their use on embedded or
otherwise resource constrained devices. In addition, we have shown that for
long sequences \emph{clustered attention} can result to almost $50\%$ reduction
in GPU training time which translates to equal reduction in CO2 emmisions and
energy consumption.

\section*{Acknowledgements}
Apoorv Vyas was supported by the Swiss National Science Foundation under grant
number FNS-30213 "SHISSM". Angelos Katharopoulos was supported by the Swiss
National Science Foundation under grant numbers FNS-30209 "ISUL" and FNS-30224
"CORTI".

\bibliographystyle{icml2020}
\bibliography{references}

\appendix
\onecolumn
\standalonetitle{Supplementary Material for \\
                 Fast Transformers with Clustered Attention}

\section{Scaling Attention with Fast Clustering}
In this section we present graphical illustrations for the
proposed \emph{clustered} and \emph{i-clustered}
attention models in \S~\ref{subsec:supp_clustered_attn}
and \S~\ref{subsec:supp_improved_attn} respectively.

\subsection{Clustered attention} \label{subsec:supp_clustered_attn}
In figure \ref{fig:supp_clustered_fc}, we present the steps involved in
\emph{clustered} attention computation for an example sequence with $8$ queries
and the number of clusters set to $3$. We first cluster the queries $Q$
using the K-means clustering to output $S$ which indicates the membership of queries to
different clusters. We use different colors to represent different clusters.
After clustering, the centroids $Q^c$ are used to compute the attention weights
$A^c$ and the new values $V^c$ for the centroids. Finally, the values are 
broadcasted to get the new values $\hat{V}$ corresponding to each query.

\begin{figure}[h]
    \includegraphics[width=\linewidth]{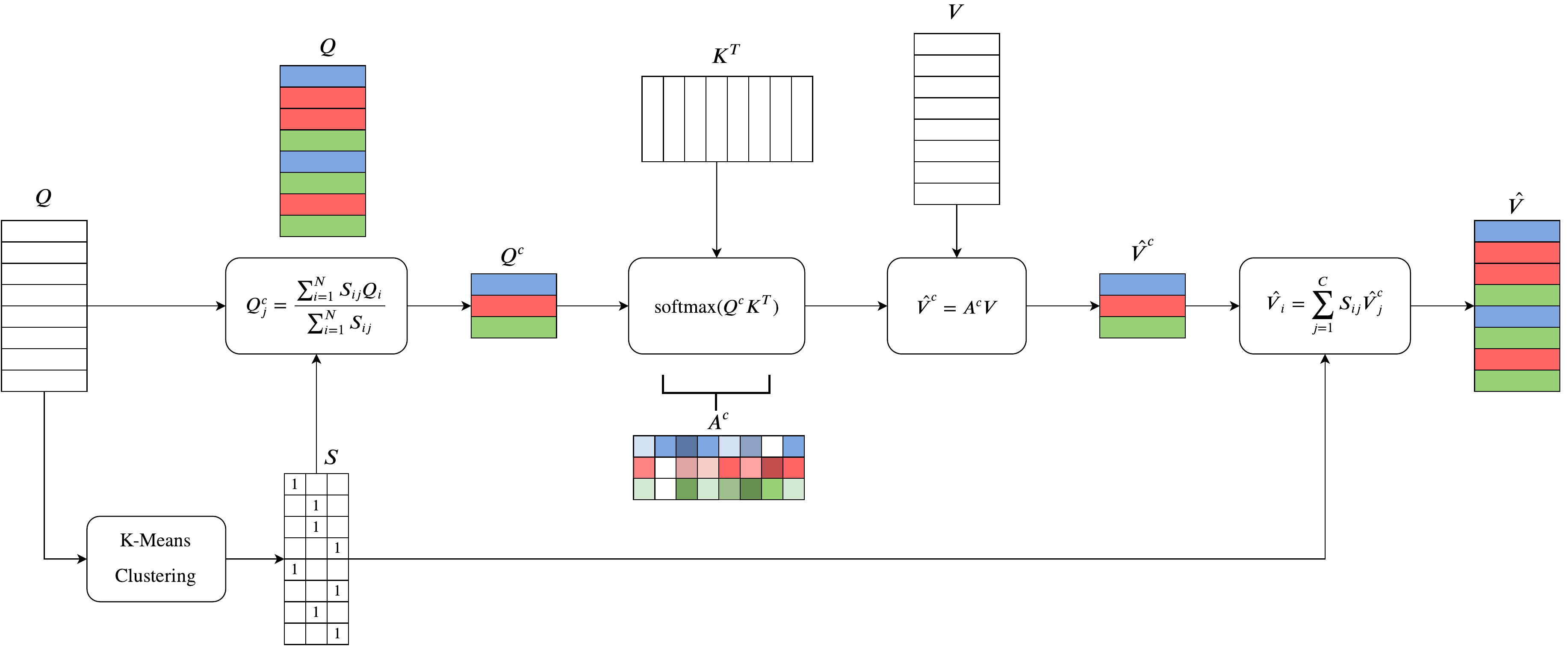}
    \caption{
        Flow-chart demonstrating the compuation for \emph{clustered} attention.
        We use different colors to represent the query groups and the computed
        centroids. The same colors are then used to show the attention weights
        $A^c$, new values for the centroids $\hat{V}^c$, and the resulting
        values $\hat{V}$ after broadcasting. For more details refer to
        \S~\ref{subsec:supp_clustered_attn} or \S~{3.2} in the main paper.
    }
    \label{fig:supp_clustered_fc}
\end{figure}

\subsection{Improved clustered attention} \label{subsec:supp_improved_attn}
In this section, we first describe how we can efficiently compute the 
\emph{i-clustered} attention using sparse dot products with the 
top-$k$ keys and values. We then present the flow chart demonstrating the
same.

As discussed in the \S~{3.3} of the main paper, the improved attention matrix
approximation $A^t_i$ for the query, $Q_i$  belonging to the cluster $j$ is
computed as follows:
\begin{equation}
    A^t_{il} = \begin{cases}
                   \frac{\hat{m}_j \expo{Q_i K_l^T}}
                             {\sum_{r=1}^N T_{jr} \expo{Q_i K_r^T}}
                             & \text{if } T_{jl} = 1 \\
                   A^c_{il} & \text{otherwise}
               \end{cases}, \label{eq:supp_improved_attention}
\end{equation}
where, $T \in \{0, 1\}^{C \times N}$, stores the top-$k$ keys for each cluster.
$T_{ji} = 1$ if the $i$-th key is among the top-$k$ keys for the $j$-th cluster
and 0 otherwise.

As described in the main paper, $\hat{m}_j$ is the total probability mass on
the top-$k$ keys for the $j$-th cluster given by:
\begin{align}
    \hat{m}_j = \sum_{r=1}^{N}T_{jr}A^c_{jr}.
\end{align}
Note that we can compute the attention weights $A^t_i$ on the top-$k$ keys by
first taking sparse dot-product of $Q_i$ with the top-$k$ keys followed
by the softmax activation and rescaling with total probablity mass $m_j$. For
the rest of the keys, the attention weight is the clustered-attention weight
$A^c_i$.

Similarly, the new values $\hat{V}_i$ can be decomposed into the following
two terms,
\begin{align}
   \hat{V}_i = \hat{V}^{t}_{i} + \hat{V}^{b}_{i}, \label{eq:supp_improved_values_fast}
\end{align}
where $\hat{V}^{t}_{i}$ is weighted average of the values corresponding to the
top-$k$ keys with weights being the improved attention on the top-$k$ keys.
$\hat{V}^b_i$ is the weighted average of the rest of the values with weights being
the clustered attention $A^c_i$.  The following equations show how we compute
$\hat{V}^t_i$ and $\hat{V}^b_i$,
\begin{equation}
   \hat{V}^{t}_{i} = \sum_{l=1}^{N}T_{jl} A^t_{il} V_l, \label{eq:supp_improved_values_top}
\end{equation}
\begin{equation}
   \hat{V}^{b}_{i} = \sum_{l=1}^{N}(1-T_{jl}) A^c_{il} V_l, \label{eq:supp_improved_values_bottom}
\end{equation}
Note that $\hat{V}^t_i$ is weighted average of $k$ values for each query and thus requires
$\bigO{N k D_v}$ operations. $\hat{V}^b_i$ only needs to be computed once per-cluster centroid
and thus requires $\bigO{N C D_v}$ operations.

In figure \ref{fig:supp_topk_fc} we present the \emph{i-clustered} attention
computation for the same example sequence with $8$ queries and the number of
clusters and top-$k$ keys set to 3. The lower half of the figure shows the new
value $\hat{V}^t$ computed by first taking sparse dot-products with the top 3 keys to
get the attention weights. This is followed by taking the weighted average of
the 3 correponding values. The top half of the figure shows the $\hat{V}^b$
computation. This is same as clustered attention computation but with attention
weights corresponding to top $3$ keys set to $0$ for $A^c$. The resulting
values $\hat{V}$ is the sum of $\hat{V}^b$ and $\hat{V}^t$.

\begin{figure}[h]
    \includegraphics[width=\linewidth]{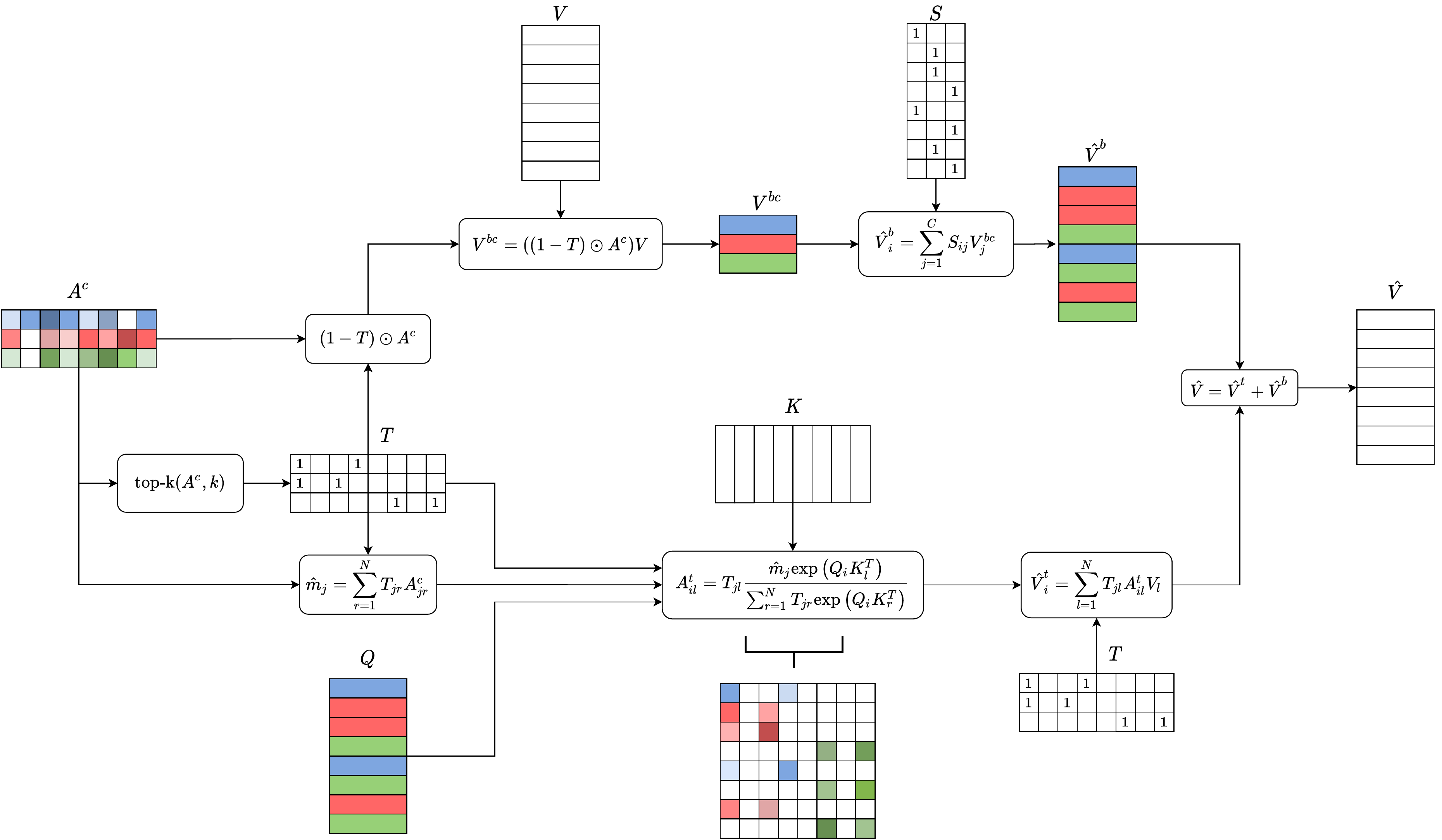}
    \caption{
        Flow-chart demonstrating the compuation for \emph{i-clustered}
        attention. The lower half of the figure shows the new value $\hat{V}^t$
        computed by sparse dot-products with the keys $K$ and values $V$
        corresponding to the the top-$k$ keys in $T$. The top half of the
        figure shows the computation for $\hat{V}^b$ which is the weighted average of
        the rest of the values with weights coming from the clustered attention
        $A^c$. The resulting values $\hat{V}$ is the sum of $\hat{V}^b$ and $\hat{V}^t$.
        For more details refer \S~\ref{subsec:supp_improved_attn} or to the \S~{3.3}
        in the main paper.
    }
    \label{fig:supp_topk_fc}
\end{figure}

\section{Quality of the approximation}\label{subsec:supp_quality_improved}

\begin{theorem} \label{thm:supp_improved_attention}
For the $i$-th query belonging to the $j$-th cluster, the improved clustered
attention $A^t_i$ and clustered attention $A^c_j$ relate to the full attention
$A_i$ as follows,
\begin{equation}
    \norm{A^t_i - A_i}_1 \leq \norm{A^c_j - A_i}_1
\end{equation}
\end{theorem}

\begin{proof}

As discussed before, the improved attention matrix approximation $A^t_i$ for
the query, $Q_i$ is computed as follows:
\begin{equation}
    A^t_{il} = \begin{cases}
                   \frac{\hat{m}_j \expo{Q_i K_l^T}}
                             {\sum_{r=1}^N T_{jr} \expo{Q_i K_r^T}}
                             & \text{if } T_{jl} = 1 \\
                   A^c_{il} & \text{otherwise}
               \end{cases}, \label{eq:supp_improved_attention_2}
\end{equation}

where, $T \in \{0, 1\}^{C \times N}$, stores the top-$k$ keys for each cluster,
$T_{ji} = 1$ if the $i$-th key is among the top-$k$ keys for the $j$-th cluster
and 0 otherwise. $\hat{m}_j$ is the total probability mass on the
top-$k$ keys for the $j$-th cluster, computed as follows:
\begin{align}
    \hat{m}_j = \sum_{r=1}^{N}T_{jr}A^c_{jr}.
\end{align}

Given the full attention $A_i$, equation \ref{eq:supp_improved_attention_2} can be
simplified to 
\begin{align}
     A^t_{il} &= \begin{cases}
                    \frac{\hat{m}_j}{m_i} A_{il}
                              & \text{if } T_{jl} = 1 \\
                    A^c_{il} & \text{otherwise}  \\
                 \end{cases},
                 \label{eq:supp_simplified_improved_attention}
\end{align}
where, $m_i$ is the total probability mass on the same top-$k$ keys for the
$i$-th query, computed using the true attention $A_i$, as follows:
\begin{align}
    m_i &= 
           \frac{\sum_{r=1}^N T_{jr} \expo{Q_i K_r^T}}
                    {\sum_{r=1}^N \expo{Q_i K_r^T}} \\
        &= \sum_{r=1}^{N}T_{jr}A_{ir}.
\end{align}


Without loss of generality, let us assume, $T_{jl}=1 \quad \forall \quad l \in
\{1,\dots,k\}$ and $T_{jl}=0 \quad \forall \quad l \in \{k+1,\dots,N\}$.

In this case, equation \ref{eq:supp_simplified_improved_attention} can be written
as:
\begin{align}
     A^t_{il} &= \begin{cases}
                    \frac{\hat{m}_j}{m_i} A_{il}
                              & \text{if} \quad l \leq k \\
                    A^c_{il} & \text{if} \quad l \geq k+1 \\
                 \end{cases}.
                 \label{eq:supp_sp_improved_attention} 
\end{align}
The total probability masses on the top-$k$ keys, $m_i$ and $\hat{m}_j$ can
now be expressed as:
\begin{align}
      m_i  &= \sum_{r=1}^{k}A_{ir}. \\
      \hat{m}_j  &= \sum_{r=1}^{k}A^c_{jr}.
\end{align}

From equation \ref{eq:supp_sp_improved_attention} it is clear that the clustered
attention, $A^c_{i}$, and the improved clustered attention, $A^t_i$, only
differ on the keys $\{1,\dots,k\}$.  Thus, it suffices to show that $A^t_i$ has
lower approximation error on these keys.  The approximation error on the
top-$k$ keys $\{1,\dots,k\}$, let it be $e_t$, between the
\emph{i-clustered} attention and the \emph{full} attention is as
follows:
 
\begin{align}
e_t &=  \sum_{l=1}^{k} \abs{A_{il} - A^t_{il}} \\
    &= \sum_{l=1}^{k} \abs{A_{il} - A_{il} \frac{\hat{m}_j}{m_i}} \\
    &= \sum_{l=1}^{k} A_{il}\abs{1 - \frac{\hat{m}_j}{m_i}} \\
    &= \abs{1 - \frac{\hat{m}_j}{m_i}} \sum_{l=1}^{k} A_{il} \\
    &= m_i\abs{1 - \frac{\hat{m}_j}{m_i}} \\
    &= \abs{m_i - \hat{m}_j} \\
    &= \abs{\sum_{l=1}^{k}A_{il} - A^c_{jl}} \\
    &\leq \sum_{l=1}^{k}\abs{A_{il} - A^c_{jl}}
\end{align}

Therefore,
\begin{align}
\norm{A_i - A^t_i}_1
    &=  \sum_{l=1}^{k} \abs{A_{il} - A^t_{il}} +
         \sum_{l=k+1}^{N}\abs{A_{il} - A^t_{il}} \\
    &=  \sum_{l=1}^{k} \abs{A_{il} - A^t_{il}} +
        \sum_{l=k+1}^{N}\abs{A_{il} - A^c_{jl}} \\
    &\leq \sum_{l=1}^{k}\abs{A_{il} - A^c_{jl}} +
        \sum_{l=k+1}^{N}\abs{A_{il} - A^c_{jl}} \\
    &\leq \norm{A_i - A^c_i}_1
\end{align}
\end{proof}

\section{Experiments}

\subsection{Time and Memory Benchmark} \label{subsec:supp_benchmark}

To measure the computational cost, we compare the memory consumption and
computation time on artificially generated sequences of various lengths. For
clustered attention we use $100$ clusters, $63$ bits for the LSH, and $10$ Lloyd
iterations for the K-Means. For the improved clustered attention, we use the
same configuration with $k=32$. For Reformer, we evaluate on two variants
using $1$ and $4$ rounds of hashing. All models consist of $1$ layer with
$6$ attention heads, embedding dimension of $64$ for each head, and a
feed-forward dimension of $1536$.

In this experiment, we measure the required memory and GPU time \emph{per
single sequence element} to perform a forward/backward pass for the various
self-attention models.  Figure \ref{fig:supp_benchmark} illustrates how these
metrics evolve as the sequence length increases from $N=2^9$ to $N=2^{15}$.
For a fair comparison, we use the maximum possible batch size for each method
and we divide the computational cost and memory with the number of samples in
each batch and the sequence length.

We note that, in contrast to all other methods, vanilla transformer scales
quadratically with respect to the sequence length and does not fit in GPU
memory for sequences longer than $2^{13}$ elements. All other methods scale
linearly. Clustered attention becomes faster than the vanilla transformer for
sequences with $1000$ elements or more, while improved clustered attention
surpasses it for sequences with $2000$ elements. Note that with respect to per
sample memory, both clustered and improved clustered attention perform better
than all other methods.  This can be explained by the fact that our method does
not require storing intermediate results to compute the gradients from multiple
hashing rounds as Reformer does. It can be seen, that lsh-$1$ is faster
than the improved clustered clustered attention, however, as also mentioned by
\citep{kitaev2020reformer} Reformer requires multiple hashing rounds to
generalize.

\begin{figure}[h]
    \begin{subfigure}{0.96\columnwidth}
        \centering
        \includegraphics[width=\columnwidth]{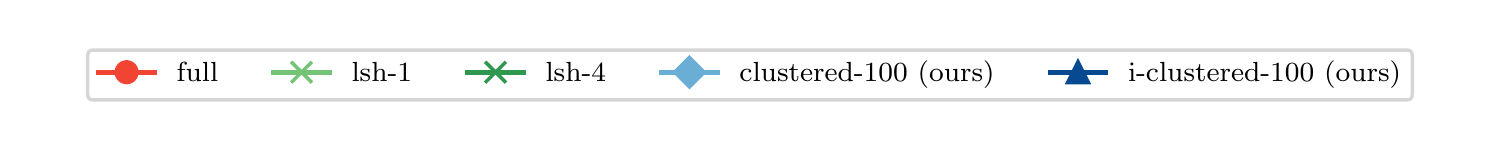}
    \end{subfigure}
    \begin{subfigure}{0.48\columnwidth}
        \centering
        \includegraphics[width=\columnwidth]{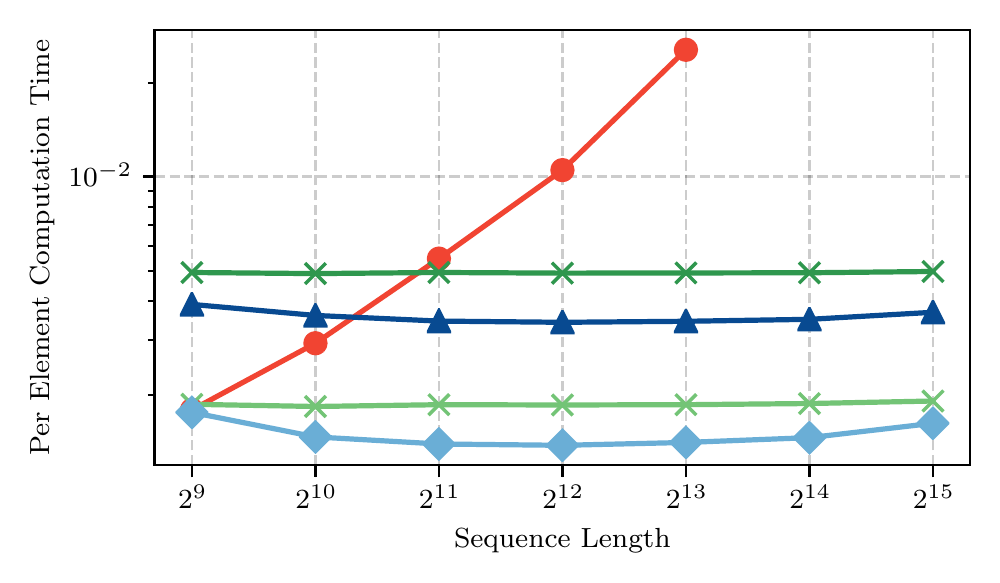}
        \caption{Per Element Time}\label{fig:supp_benchmark_time}
    \end{subfigure}
    \begin{subfigure}{0.48\columnwidth}
        \centering
        \includegraphics[width=\columnwidth]{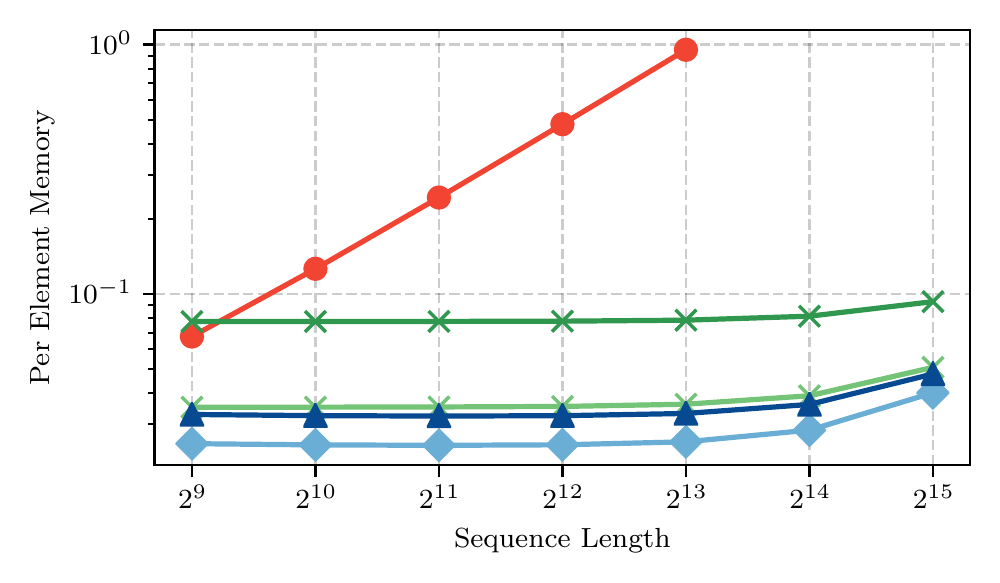}
        \caption{Per Element Memory}\label{fig:supp_benchmark_memory}
    \end{subfigure}
    \caption{Per element GPU time and memory consumption for a forward/backward
         pass. All models, except full, scale linearly with respect to the
         sequence length since they have constant time and memory per
         element. Detailed analysis can be found in
         \S~\ref{subsec:supp_benchmark}.}
    \label{fig:supp_benchmark}
\end{figure}

\subsection{Ablation on clusters and sequence length} \label{sec:supp_ablation}

\begin{figure}
    \centering
    \textbf{Accuracy with respect to clusters and hashing rounds}\\[1em]
    \begin{subfigure}[t]{0.32\columnwidth}
        \centering
        \includegraphics[width=\columnwidth]{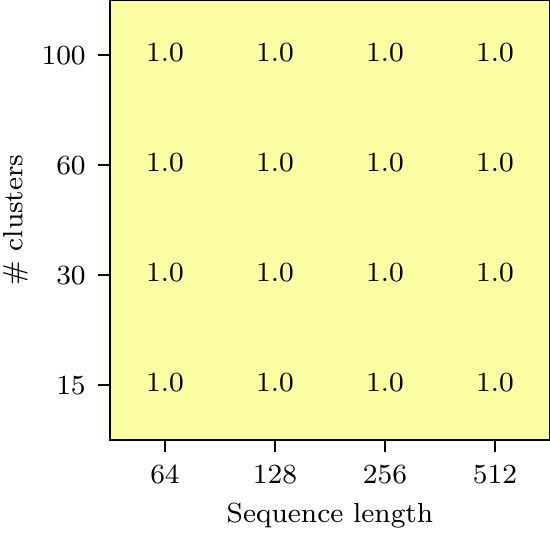}
        \caption{Improved clustered} \label{fig:supp_ablation_ic}
    \end{subfigure}\hfill
    \begin{subfigure}[t]{0.32\columnwidth}
        \centering
        \includegraphics[width=\columnwidth]{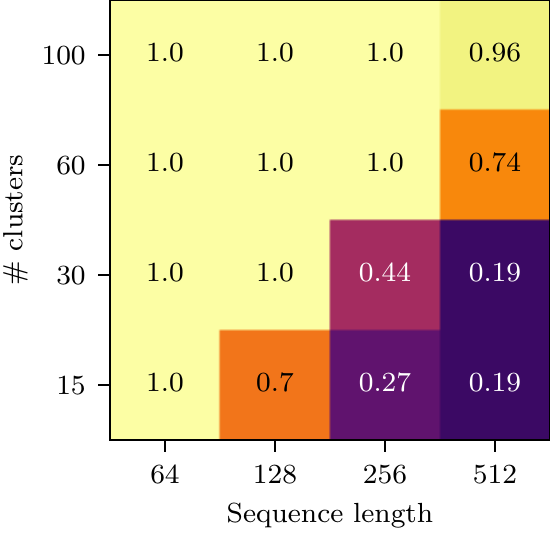}
        \caption{Clustered} \label{fig:supp_ablation_c}
    \end{subfigure}\hfill
    \begin{subfigure}[t]{0.32\columnwidth}
        \centering
        \includegraphics[width=\columnwidth]{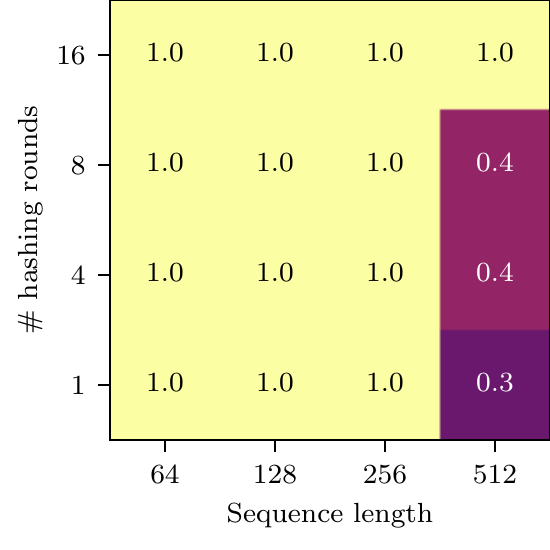}
        \caption{Reformer} \label{fig:supp_ablation_lsh}
    \end{subfigure}
    \caption{The heatmaps depict the achieved accuracy on an artificial copy
             task (\S~\ref{sec:supp_ablation}) as the sequence length, the number of
             clusters and the number of hashing rounds varies. Improved
             clustered (\ref{fig:supp_ablation_ic}) is the only fast transformer
             variant that can solve the task perfectly for any sequence length
             and number of clusters combination.}
    \label{fig:supp_ablation}
\end{figure}

Following \citep{kitaev2020reformer}, we introduce a synthetic task to analyze
the relationship between the number of clusters and sequence length. In our
task, the transformer models need to copy some symbols that are masked out from
either the first or second half of the sequence. In particular, we generate a
random sequence of tokens and we prepend a unique separator token, let it be
$0$. The sequence is then copied to get a target of the form $0w0w$, where $w
\in \{1, \dots, C\}^L$, $C$ is the number of possible symbols and $L$ is the
sequence length.  To generate the input, we replace some symbols from the first
half of the sequence and some different symbols from the second half, such that
the target sequence can be reconstructed from the input. An example of an input
output pair with $L=4$ can be seen in figure \ref{fig:supp_masked_copy}. Note that
to solve this task, transformers simply need to learn to attend to the
corresponding tokens in the two identical halves of the sequence.
\bgroup
\renewcommand{\arraystretch}{1.2}
\begin{figure}[H]
    \centering
    \begin{tabular}{r|c|c|c|c|c|c|c|c|c|c|}
        \cline{2-11}
        \textbf{Input}  & 0 & 4 & M & 2 & 2 & 0 & 4 & 5 & M & 2 \\
        \cline{2-11}
        \textbf{Output} & 0 & 4 & 5 & 2 & 2 & 0 & 4 & 5 & 2 & 2 \\
        \cline{2-11}
    \end{tabular}
    \caption{Example of an input and output pair for the masked copy task. M
    denotes the masked out tokens.}
    \label{fig:supp_masked_copy}
\end{figure}
\egroup
We set the sequence length $L$ to one of $\{31, 63, 127, 255\}$ which means
the input length varies between $N=2^6$ and $N=2^9$. For each sequence, we
sample tokens uniformly from $\{1, \dots, 10\}$ and randomly mask out $20\%$ of
the tokens. To analyze the impact of number of clusters on performance, we
train full transformer as well as clustered variants with different number of
clusters and Reformer with different number of hashing rounds.

All transformer variants consist of $4$ layers, $6$ attention heads, embedding
dimension of $32$ for each head, and feed-forward dimension of $768$. For both
clustered and improved clustered attention, we set the number of bits for LSH
to $63$ and the number of Lloyd iterations for the K-Means to $10$. Both
clustered and improved clustered attention are trained with $15$, $30$, $60$
and $100$ clusters. We also train Reformer with $1$, $4$, $8$ and $16$ hashing
rounds. Finally, all models are trained using R-Adam optimizer
\cite{liu2020radam} with a learning rate of $0.0002$, batch size of $32$ for
$5000$ iterations.

In figure \ref{fig:supp_ablation}, we illustrate the results of this experiment as
heatmaps depicting the achieved accuracy for a given combination of number of
clusters and sequence length for clustered transformers and number of hashing
rounds and sequence length for Reformer. Note that the vanilla transformer
solves the task perfectly for all sequence lengths. We observe that both
clustered (Fig.~\ref{fig:supp_ablation_c}) and Reformer
(Fig.~\ref{fig:supp_ablation_lsh}) require more clusters or more rounds as the
sequence length increases. However, improved clustered achieves the same
performance as vanilla transformers, namely \emph{perfect accuracy}, for every
number of clusters and sequence length combination. This result increases our
confidence that the required number of clusters for our method is not a
function of the sequence length but of the task at hand.

\subsection{Automatic Speech Recognition}
In this section, we present the details for the ASR experiments such as
transformer architecture, optimizer and learning rate schedule. As mentioned
in the main paper, for \emph{i-clustered}, unless specified, $k$ is set to 32.
Furthermore, all transformers have $6$ heads with an embedding dimension of
$32$ on each head and feed-forward dimension of $768$. Other architectural
details specific to experiments are described later.

\subsubsection{Wall Street Journal}

\boldparagraph{Convergence Behaviour}  \label{subsec:supp_wsj_convergence}

For this experiment, we train transformer with full, clustered and Reformer
attention variants. All models consist of $9$ layers. For
Reformer, we train two variants with $1$ and $4$ rounds of hashing with chunk
size fixed to $32$ as suggested. For clustered and improved clustered attention
we set the number of clusters to $100$. We also set the number of Lloyd
iterations for K-Means to $10$ and the bits for LSH to $63$.  All models are
trained to convergence using the R-Adam optimizer \cite{liu2020radam} with a
learning rate of $0.0001$, max gradient norm set to $10.0$ and and weight decay
of $0.01$. The learning rate is dropped when the validation loss plateaus. For
each model we select the largest batch size that fits the GPU.  The \emph{full}
attention model was trained with a batch size of $2$ while the clustered
variants: \emph{clustered} and \emph{i-clustered} could fit batch sizes of $14$
and $10$ respectively.  For Reformer variants: \emph{lsh-$1$} and
\emph{lsh-$4$}, batch sizes of $8$ and $6$ were used.

\begin{figure}[h]
    \centering
    \begin{subfigure}{0.45\columnwidth}
        \includegraphics[width=\columnwidth]{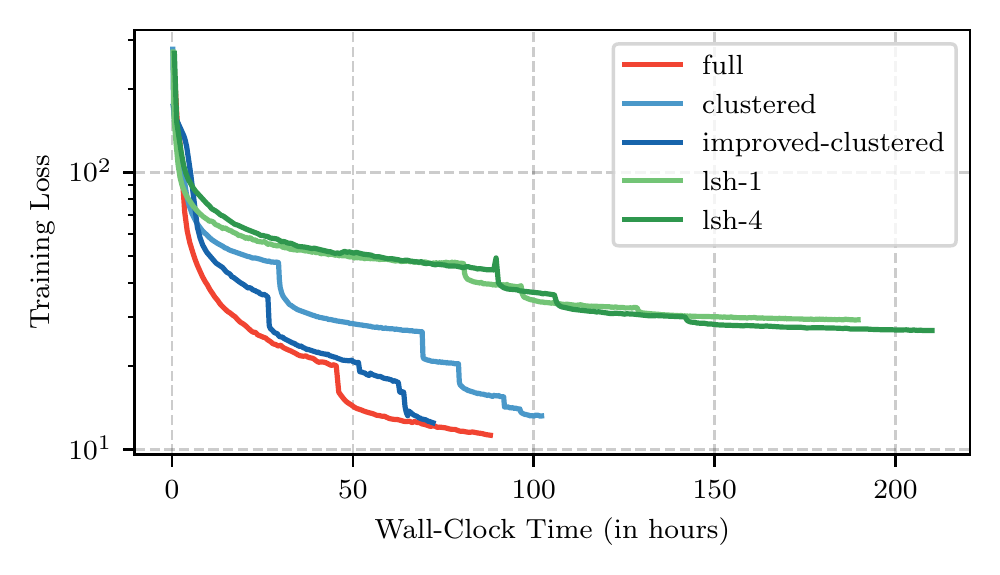}
        \caption{Wall Street Journal} \label{fig:supp_wsj_convergence}
    \end{subfigure}
    \begin{subfigure}{0.45\columnwidth}
        \includegraphics[width=\columnwidth]{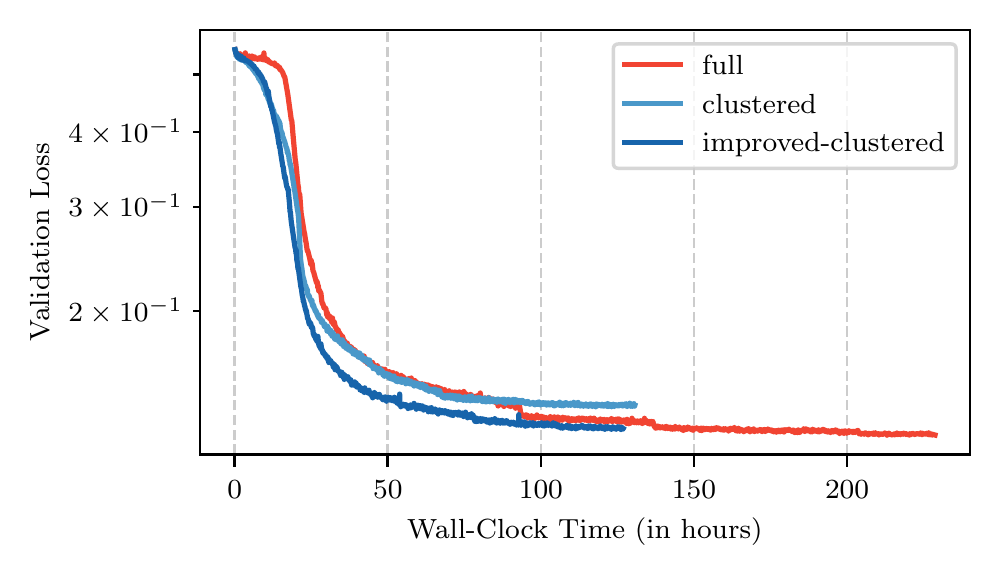}
        \caption{Switchboard} \label{fig:supp_swbd_convergence}
    \end{subfigure}
    \caption{
        We show training/validation loss convergence for different transformer
        variants. Only \emph{i-clustered} has a faster or comparable wall-clock
        convergence to full attention. Both the clustered variants are have a
        significantly better convergence than both \emph{lsh-$1$} and
        \emph{lsh-$4$}. Note that due to a smaller batch size \emph{full} makes
        many more updates than all other transformer variants. More details can
        be found in \S~\ref{subsec:supp_wsj_convergence} and
        \S~\ref{subsec:supp_swbd_convergence}.
    }
\end{figure}

In figure \ref{fig:supp_wsj_convergence}, we show the training loss convergence for
different transformer variants. It can be seen that \emph{i-clustered} has a
much faster convergence than the  \emph{clustered} attention. This shows that
the improved clustered attention indeed approximates the full attention better.
More importantly, only the \emph{i-clustered} attention has a comparable
wall-clock convergence.  Given that \emph{full} has a much smaller batch size,
it make many more updates per-epoch. We think that a slightly smaller
batchsize with more updates would have been a better choice for the clustered
transformers w.r.t. the wall-clock convergence. This is reflected in the
Switchboard experiments where the batchsizes for clustered variants were
smaller due to more layers.  Finally, as can be seen from the wall-clock
convergence, the clustered transformers significantly outperform the Reformer
variants.

\boldparagraph{Speed-Accuracy Tradeoff}  \label{subsec:supp_wsj_finetune}

As described in the main paper, for this task we additionally train \emph{full}
with $4$ and $6$ layers. Similary, we train \emph{clustered} with 9 layers, and
$200$ and $300$ clusters. We also train an \emph{i-clustered} model with 9 layer
and $200$ clusters, and smaller models with 6 layers, and $100$ and
$200$ clusters.

For \emph{clustered} and \emph{i-clustered} variants with 9 layers, we
finetuned the previously described models trained with $100$ clusters. We
finetuned for $15$ epochs with a learning rate of $0.00001$. We train \emph{full} with
$4$ and $6$ layers to convergence in a similar fashion to the \emph{full} with
$9$ layers described previously. Finally, for \emph{i-clustered}, we first
trained model with $6$ layers and $100$ clusters using the training strategy
used for $9$ layers and $100$ clusters. We then finetuned this model for $15$
epochs using $200$ clusters and a learning rate of $0.00001$.

\subsubsection{Switchboard}

\boldparagraph{Convergence Behaviour}  \label{subsec:supp_swbd_convergence}

For this experiment, we train transformer with full and clustered attention
variants. All models consist of $12$ layers. For clustered and
improved clustered attention we set the number of clusters to $100$. We also
set the number of Lloyd iterations for K-Means to $10$ and the bits for LSH to
$63$.

Following common practice for flat-start lattice-free MMI training, we train
over multiple gpus with weight averaging for synchronization as described in
\cite{povey2015parallel}.  Specfically, we modify the \textit{e2e} training
recipe for the \textbf{Wall Street Journal} in Kaldi \cite{povey2011kaldi} with
the following two key differences: first, the acoustic model training is done
in PyTorch and second, we use R-Adam optimizer instead on natural stochastic
gradient descent.

All models are trained using the R-Adam optimizer with a learning rate of
$0.0002$, max gradient norm set to $10.0$ and and weight decay of $0.01$. The
learning rate is dropped when the validation loss plateaus. We use the word
error rate (WER) on the validation set for early stopping and model selection.
The \emph{full} attention model is trained with a batch size of $2$
while the clustered variants: \emph{clustered} and \emph{i-clustered} are
trained with a batch size of $6$.

In figure \ref{fig:supp_swbd_convergence}, we show the training loss convergence for
different transformer variants. It can be seen that \emph{i-clustered} has the
fastest convergence for this setup. Note that the overall training time for
\emph{clustered} attention is still less than that of \emph{full} as it starts
to overfit early on the validation set WER.

\boldparagraph{Speed-Accuracy Tradeoff}  \label{subsec:supp_swbd_finetune}

For this task we additionally train \emph{full} with $6$ and $8$ layers.
Similary, we train \emph{clustered} with 12 layers, and $200$ and $300$
clusters. We also train \emph{i-clustered} with 12 layer and $200$ clusters,
and smaller models with 8 layers, and $100$ and $200$ clusters.

For \emph{clustered} and \emph{i-clustered} variants with 12 layers, we
finetuned the previously described models trained with $100$ clusters. We
finetuned for $5$ epochs with a learning rate of $0.00001$. Once again,
\emph{full} with $6$ and $8$ layers were trained to convergence similar to
\emph{full} with $12$ layers described previously. Finally, for
\emph{i-clustered}  with $8$ layers, we first train a model with $100$ clusters
using the training strategy used for $12$ layers and $100$ clusters. We then
finetuned this model for $5$ epochs using $200$ clusters and a learning rate
of $0.00001$.

\subsection{RoBERTa Approximation} \label{subsec:supp_approximation}

In this section we provide a qualitative comparison between the
\emph{full} attention, and the clustered attention variants \emph{clustered}
and \emph{i-clustered} used for approximation. As described in main paper,
we use $25$ clusters for both attention variants. In Figure \ref{fig:attn_qualitative}
we show the attention distribution for the question tokens for a randomly
selected question-context tuple from the SQuAD dataset. For each token in the
question we show the attention distribution over the input sequence formed by
concatenating question and context tokens with \emph{CLS} and \emph{SEP} tokens
appended. It can be seen that with only few clusters, improved clustered
approximates the full attention very closely even when the attention
distribution has complicated and sparse patterns. In contrast, clustered
attention fails to capture such attention distribution during approximation.
Moreover, it can further be seen that for almost all question tokens, both full
and improved clustered have the same tokens with the highest attention weights.
This further strengthens our believe that improved clustered attention can
approximate a wide range of complicated attention patterns.

\begin{figure}[h]
    \centering
    \begin{subfigure}{\columnwidth}
        \FramedBox{3.8cm}{0.93\columnwidth}
        {Manning finished the year with a career-low
        67.9 passer rating, throwing for 2,249 yards and nine touchdowns, with
        17 interceptions. In contrast, Osweiler threw for 1,967 yards, 10
        touchdowns and six interceptions for a rating of 86.4. Veteran receiver
        \textcolor{red}{Demaryius Thomas} led the team with 105 receptions for 1,304 yards and
        six touchdowns, while Emmanuel Sanders caught 76 passes for 1,135 yards
        and six scores, while adding another 106 yards returning punts. Tight
        end Owen Daniels was also a big element of the passing game with 46
        receptions for 517 yards. Running back C. J. Anderson was the team's
        leading rusher 863 yards and seven touchdowns, while also catching 25
        passes for 183 yards. Running back Ronnie Hillman also made a big
        impact with 720 yards, five touchdowns, 24 receptions, and a 4.7 yards
        per carry average. Overall, the offense ranked 19th in scoring with 355
        points and did not have any Pro Bowl selections.}
        \caption{\emph{context}} \label{fig:supp_context}
    \end{subfigure}
    \begin{subfigure}{0.8\columnwidth}
        \includegraphics[width=\columnwidth]{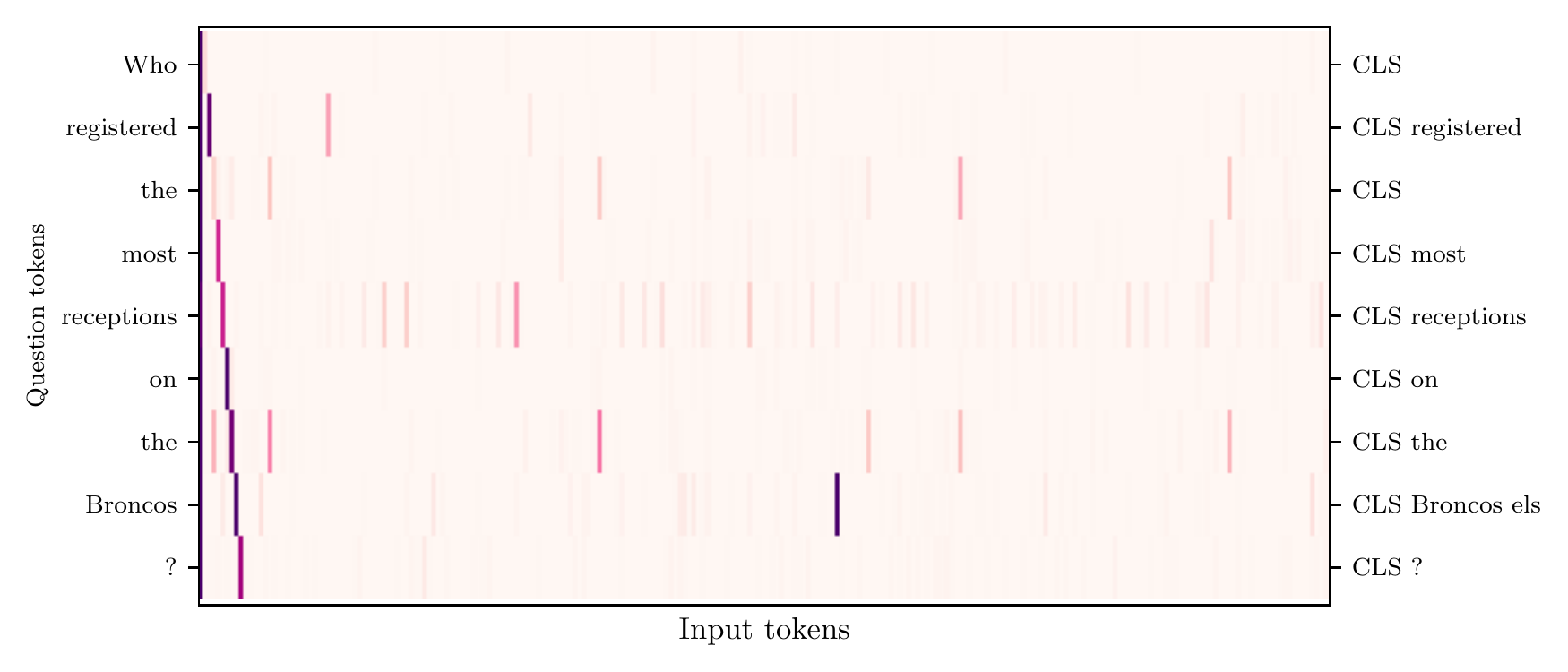}
        \caption{\emph{full}} \label{fig:supp_full_appx}
    \end{subfigure}
    \begin{subfigure}{0.8\columnwidth}
        \includegraphics[width=\columnwidth]{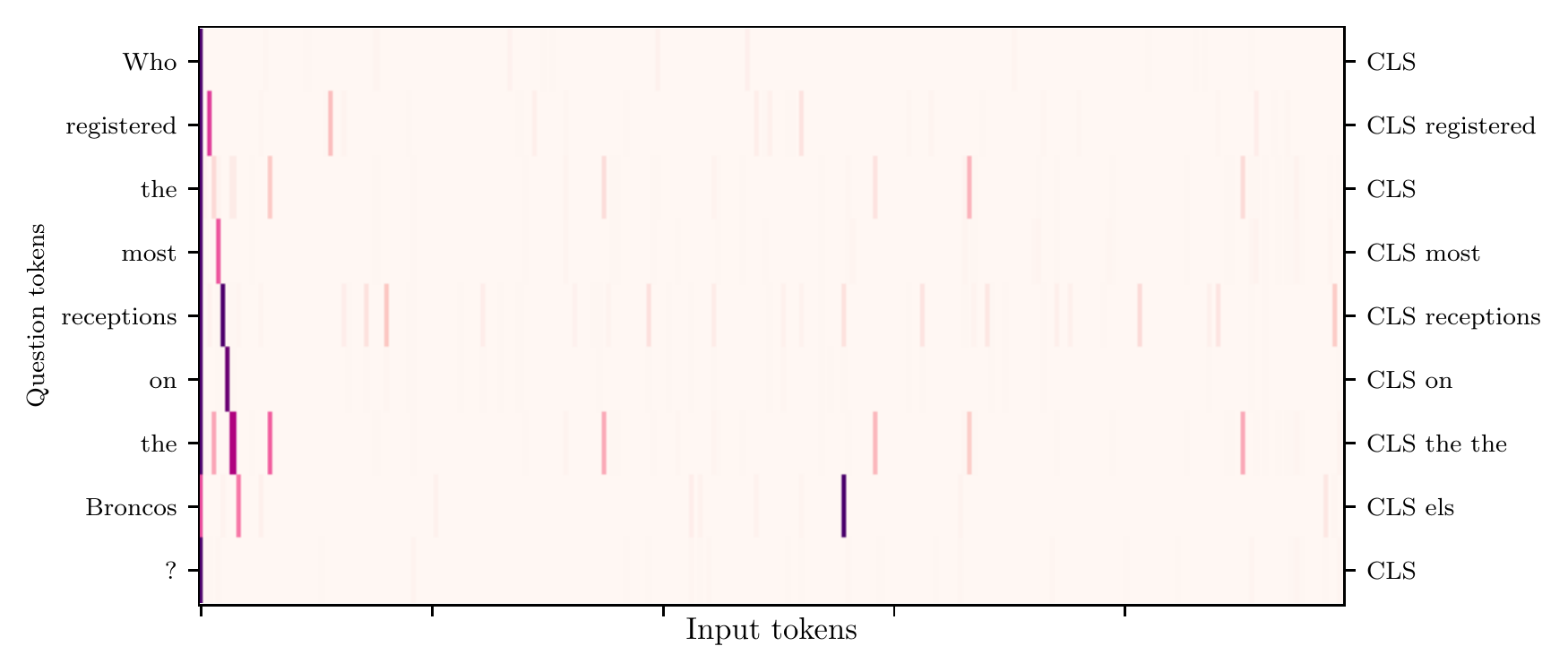}
        \caption{\emph{improved-clustered}} \label{fig:supp_iclus_appx}
    \end{subfigure}
    \begin{subfigure}{0.8\columnwidth}
        \includegraphics[width=\columnwidth]{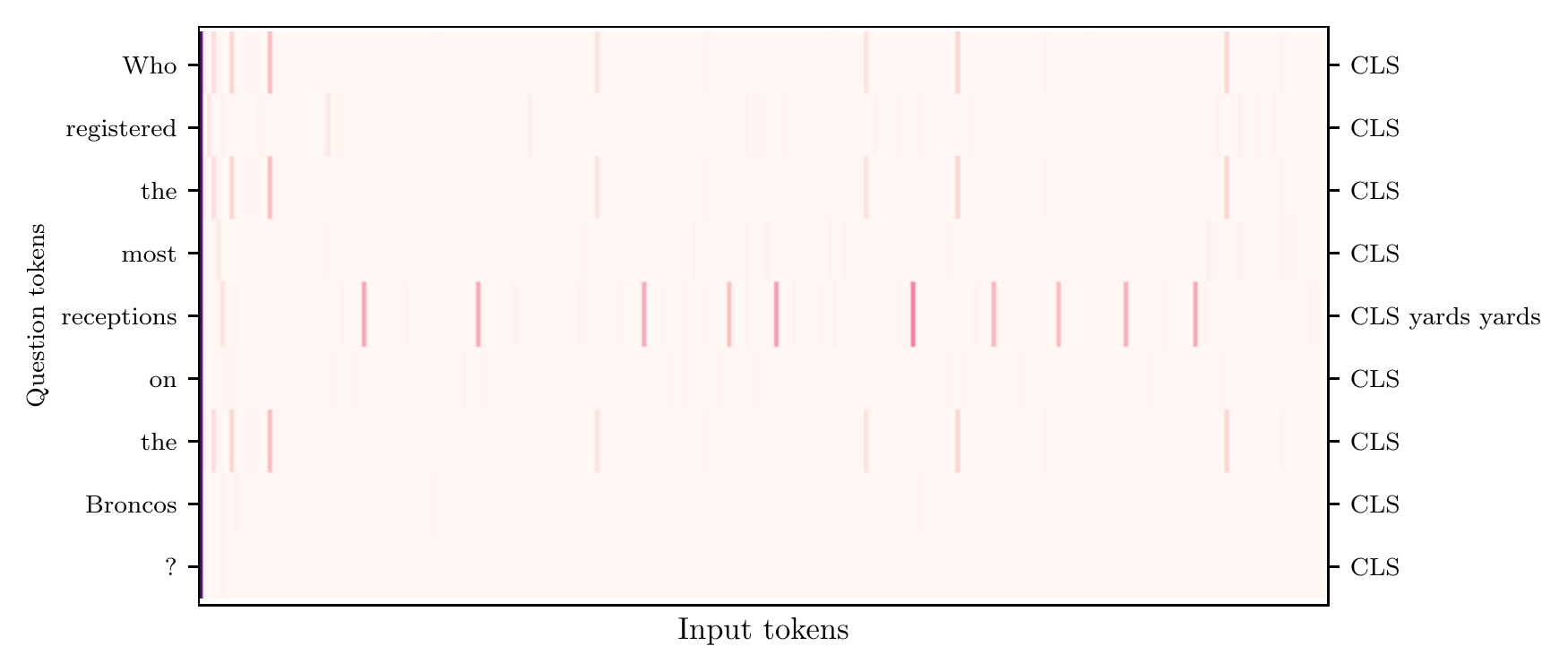}
        \caption{\emph{clustered}} \label{fig:supp_clus_appx}
    \end{subfigure}

    \caption{
        Attention matrices for question-context tuples for \emph{full}
        attention, and \emph{clustered} and \emph{i-clustered} attention used
        for approximation. \ref{fig:supp_context} shows the the context for the
        question with answer higlighted in red.  \ref{fig:supp_full_appx} shows the
        attention distribtution for \emph{full}, \ref{fig:supp_iclus_appx} and
        \ref{fig:supp_clus_appx} show the approximation using \emph{i-clustered} and
        \emph{clustered} respectively. Note that \emph{i-clustered} has
        attention patterns very similar to \emph{full} while \emph{clustered}
        shows qualitatively different attention patterns.  For each question
        token, we also present the tokens with highest attention above a
        threshold on the right axis. For more information refer to
        \S~\ref{subsec:supp_approximation}.
    }
    \label{fig:attn_qualitative}
\end{figure}

\checknbdrafts
\end{document}